\documentclass[12pt]{article}

\usepackage[margin=1in]{geometry}  
\usepackage{epsfig}
\usepackage{graphics,graphicx,color}
\usepackage{amssymb,amsfonts,amsthm,amsmath}
\usepackage{bbm}
\usepackage{multirow, booktabs}
\usepackage{mathrsfs}
\usepackage{natbib}
\usepackage{slashbox}
\usepackage{longtable,booktabs}
\usepackage{arydshln}
\usepackage{paralist}
\usepackage{algorithm}
\usepackage{algorithmic}

\newcommand{\bm}[1]{\mbox{\boldmath$ #1 $\unboldmath}}

\newtheorem{proposition}{Proposition}
\newtheorem{lamma}{Lamma}
\newtheorem{theorem}{Theorem}

\begin{document}

\title{Do-AIQ: A Design-of-Experiment Approach to Quality Evaluation of AI Mislabel Detection Algorithm}
\author{Lian, J.$^a$, Choi, K.$^b$, Veeramani, B.$^b$, Hu, A.$^b$, Freeman, L.$^a$,\\
Bowen, E.$^b$, and Deng, X$^a$\footnote{Address for correspondence: Xinwei Deng, Professor of Statistics, Co-Director of Statistics and Artificial Intelligence Laboratory at Virginia Tech (Email: xdeng@vt.edu).}
\date{%
    $^a$Department of Statistics, Virginia Tech\\%
    $^b$Deloitte \& Touche LLP\\[2ex]%
}
}

\maketitle

\begin{abstract}
The quality of Artificial Intelligence (AI) algorithms is of significant importance for confidently adopting algorithms in various applications such as cybersecurity, healthcare, and autonomous driving. This work presents a principled framework of using a design-of-experimental approach to systematically evaluate the quality of AI algorithms, named as Do-AIQ. Specifically, we focus on investigating the quality of the AI mislabel data algorithm against data poisoning. The performance of AI algorithms is affected by hyperparameters in the algorithm and data quality, particularly, data mislabeling, class imbalance, and data types. To evaluate the quality of the AI algorithms and obtain a trustworthy assessment on the quality of the algorithms, we establish a design-of-experiment framework to construct an efficient space-filling design in a high-dimensional constraint space and develop an effective surrogate model using additive Gaussian process to enable the emulation of the quality of AI algorithms. Both theoretical and numerical studies are conducted to justify the merits of the proposed framework. The proposed framework can set an exemplar for AI algorithm to enhance the AI assurance of robustness, reproducibility, and transparency.
\end{abstract}

\section{Introduction}

Great advancement has been made by machine learning algorithms on various scientific and engineering areas \citep{ma2019trafficpredict,ben2020cybersecurity}.
However, the safety and quality assurance of these algorithms are still a significant concern. For example, a small malicious perturbation on data can deceive AI algorithms and result in catastrophe when applying them to real-life practices \citep{chen2020frank}.
Therefore, systematically evaluating the quality of AI algorithms is of great importance for the assurance of using the algorithms in practice. 
In this work, we present a principled {\bf Do-AIQ} framework of using a {\bf D}esign-{\bf o}f-experiment approach to systematically evaluate the {\bf AI} {\bf Q}uality.
The proposed framework can set an exemplar for AI algorithm practitioners to enhance the AI assurance of robustness, reproducibility, and transparency.

Specifically, we focus on investigating the quality of the AI mislabel detection (MLD) algorithm against data poisoning. In classification tasks, mislabeling the responses can disturb model training and undermine the performance of algorithms by simply assigning wrong labels to training data. 
Various MLD algorithms \citep{malossini2006detecting,guan2011identifying,vu2019robust} have been developed in the literature. For example, \cite{pulastya2021assessing} proposed a structure including a variational autoencoder and a simple classifier to construct the so-called mislabeling score to distinguish the mislabeled data from normal data. It is known that the performance of MLD algorithms can be affected by two major groups of factors, hyperparameters in the algorithm and data quality factors, including the number of mislabeled observations in training data, class imbalance, and data types.
However, no comprehensive evaluation exists to assess the assurance (such as robustness, reproducibility, and transparency) of these methods. Besides, the validity of the evaluation metrics and processes is not verified. The proposed Do-AIQ approach will fill these gaps and shed a light on the assessment of MLD AI algorithms and on understanding their limitations. 

There are two major groups of factors affecting performance of the MLD algorithm. The first group is data quality factors. Typically, class imbalance is one essential data quality factor. For MLD algorithms, the number of  mislabeled observations in training data is one key factor that affects the performance of the algorithm. The data type (i.e., dataset of interest) can be another data quality factor. The second group is hyperparameters of the algorithms, such as the weights in the loss function and certain threshold values. 
To comprehensively evaluate the quality of the AI algorithms and obtain a trustworthy assessment on the quality of the algorithms, we propose a design-of-experimental framework to construct an efficient space-filling design in a high-dimensional constraint space and develop an effective surrogate model based on an additive Gaussian process to enable the emulation of the quality of AI algorithms. 
Specifically, we consider the design space consisting of both continuous and categorical factors with certain constrains. For such a high-dimensional constraint space, we adopt constraint space-filling design \citep{joseph2016space} to systematically investigate how the data quality affects the quality of AI algorithm when internal structure of AI algorithms varies. It paves a foundation for understanding the uncertainty of AI algorithms and finding optimal configuration of hyperparameters of the MLD algorithm.

The main contribution of this work is as follows. First, we propose a principled framework of using a design-of-experimental approach to systematically evaluate the quality of AI algorithms against data poisoning. The developed framework is not restricted to the MLD algorithm, but can be used for evaluating other AI algorithms, especially when the data quality and algorithm structure (i.e., hyperparameters) are intertwined. Second, to systematically investigate how the data quality affects the quality of AI algorithm when the internal structure of AI algorithms varies, we propose an effective design criterion with an efficient construction algorithm to obtain a space-filling design in a high-dimensional constraint space to investigate the quality of MLD algorithms in terms of detection accuracy and prediction accuracy. Specifically, we consider the design space consisting of three continuous factors without constraint, ten continuous factors of class proportions with linear constraint, and one binary factor for ``data type". The construction algorithm is efficient by leveraging the simplicity of coordinate descent in discrete optimization and constraint continuous optimization. Third, due to the complexity of design criterion and design space, an initial design is crucial to enable the design construction algorithm. Our method of initial design based on algebraic construction is very fast in computation with flexible run sizes. Fourth, we adopt an additive Gaussian process model as a surrogate model to emulate the quality of AI algorithms as a function of data quality factors and AI algorithm’s hyperparameters. The use of an additive Gaussian process can accommodate both continuous and categorical factors of interest, providing accurate prediction and uncertainty quantification of the quality of the algorithm. 



\section{The Proposed Do-AIQ Framework}\label{sec:prop.framework}
Figure \ref{fig:DO-AIQ} illustrates the general idea of the proposed Do-AIQ framework. It aims to investigate how various factors affect the performance of AI algorithms by a careful  design of experiment (DoE) and a proper emulator.
In this section, we will firstly describe a recent MLD algorithm in the literature. 
Then we detail the proposed Do-AIQ framework with design factors and responses in Section \ref{factor.response}, design construction in Section \ref{design.construction}, and surrogate modeling in Section \ref{sec:GPmodel}. 

\begin{figure}[h]
\begin{center}
\includegraphics[width=0.95\textwidth]{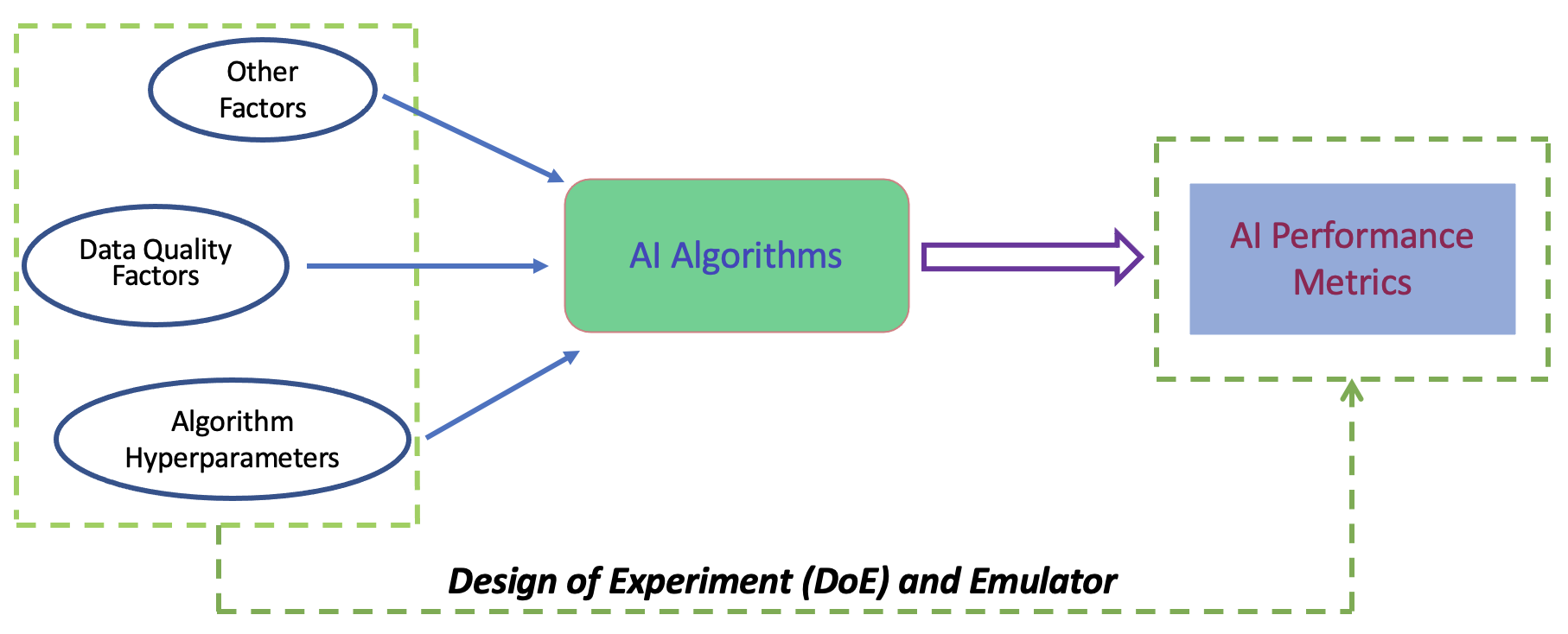}
\caption{An illustration of the proposed Do-AIQ framework.}\label{fig:DO-AIQ}
\end{center}
\end{figure}
 
\subsection{The Detection Algorithm}\label{detect.algorithm}
Without loss of generality, 
we consider the $m$-class classification problem with data $(\bm U_{i}, c_{i})$, where $U_{i}=(u_{1(i)},...,u_{r(i)})$, $u_{j(i)}$ is $j^{th}$ element of $U_{i}$, $j=1,...,r$ and $c_{i} \in \{1, \ldots, m\}$.
Here we use \cite{pulastya2021assessing} as an example of the MLD algorithm to present the proposed framework on how to comprehensively evaluate the quality of the AI algorithms. This algorithm considered a structure including one variational autoencoder (VAE) and one simple sigmoid classifier to extract mislabeling signals, as shown in Figure \ref{fig:vae}.
The encoder is based on convolutional layers and produces two dense layers to encode mean and variance of the latent layer $\bm v$ of 100 dimensions. Then they put latent layer $\bm v$ into the decoder to reconstruct the image. The latent layer $\bm v$ is also leveraged as the input of the simple sigmoid classifier. 
For this AI algorithm, they use a composite loss function as
\begin{align}
    L=w L_{ELBO} + (1-w) L_{CL},
    \nonumber
\end{align}
where $0 \le w \le 1$ is the weight, the $L_{ELBO} = E_{q(\bm v|\bm U)}[\log(p(\bm U| \bm v))] - KL[q(\bm v| \bm U)||p(\bm v)]$ is the evidence lower bound loss (ELBO) function for the VAE, and $L_{CL}= -\sum t(c_i) log(p(c_{i}))$ is the cross entropy loss function for the classifier.
Here $KL[\cdot]$ is the Kullback-Leibler divergence, $t(\cdot)$ is the true distribution, and $p(\cdot)$ is the predicted distribution. 
Note that the VAE and the classifier are trained simultaneously.  
Based on the trained structure, the so-called ``mislabeling score" is constructed as follows. 
For a group of given class $\{U_{i}:\forall i~ with~c_{i}=l\}$, one calculates the median of absolute deviation from the reconstructed median as $M_l = median\{|U_{recon_i}-m_l|: \forall i \mbox{ with } c_i=l \}$, where $m_l$ is the median of all reconstructed images for the group. Count element-wisely if the deviation from reconstructed image to the median is greater than $\alpha \times M_l$ (i.e., $\sum_{j=1}^{r}\mathbbm{1}{(|u_{j(recon_i)}-m_{j(l)}|> \alpha M_{j(l)})}$, where $\mathbbm{1}(\cdot)$ is an indicator function; $u_{j(recon_i)}$, $m_{j(l)}$, and $M_{j(l)}$ correspond to the $j^{th}$ element of $U_{recon_i}$, $m_l$, and $M_l$).
If the count is greater than a threshold, the algorithm marks the item $i$ as mislabeled. It is seen that $\alpha$ is a hyperparameter.

\begin{figure}
\begin{center}
\includegraphics[width=0.95\textwidth]{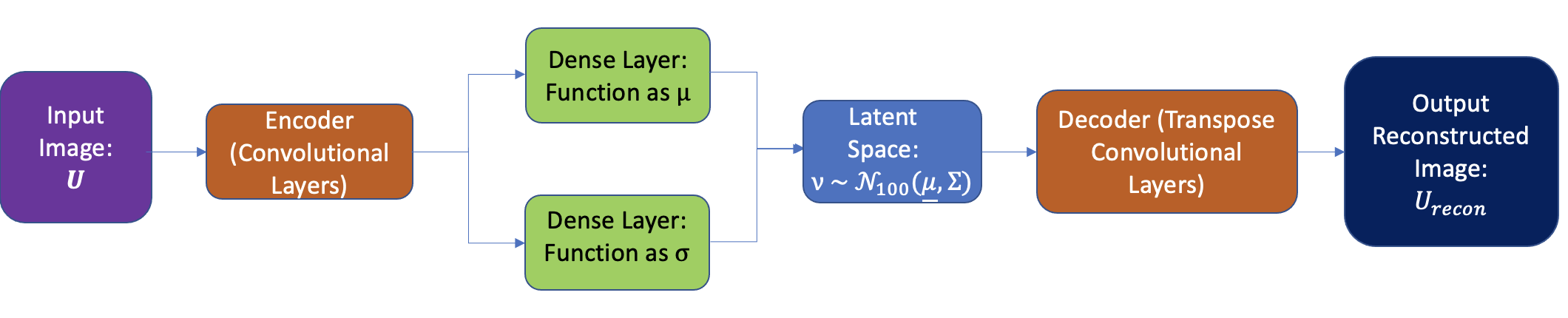}
\caption{The structure based on variational autoencoder (VAE) trained to extract invariant features for given classes as an prerequisite to construct a mislabel score.}\label{fig:vae}
\end{center}
\end{figure}

\subsection{Design Factors and Response Variables}\label{factor.response}

As shown in Figure \ref{fig:DO-AIQ}, the quality of the MLD algorithm is affected by data quality factors, algorithm hyperparameters, and other factors, such as degradation of hardware, and stability of the internet. 
In this work, we focus on investigating the effects of data quality factors and algorithm hyperparameters.

For the MLD algorithm,
one influential factor is the weight in the loss function. 
The weights ratio $\frac{w}{1-w}$ is considered as one factor, denoted as $z_1$. The value of $\alpha$ as the threshold of deviation is also critical for detecting the mislabeled data points. Consequently, we consider the hyperparameter $\alpha$ as another factor of interest, denoted as $z_{2}$. For the data quality factors, we mainly consider the proportion of mislabeling in the training data, the class imbalance in training data, and the type of datasets.
First, the percentage of mislabeled data is considered as an important factor, denoted as $z_3$.
Intuitively, when the proportion of mislabeled data are high in training data, it will be difficult to extract informative features with inaccurate information on responses.  
Besides, the class imbalance can undermine both detection accuracy and classification accuracy since sufficient information about the minority classes could be absent. It is important to investigate the robustness of the MLD algorithm with respect to the class imbalance in training data as well as the proportions of mislabeled data. 
Thus, we assume that the class imbalance corresponds to proportions of classes in the training set as $x_{1},x_{2},..,x_{m}$ with constraint $\sum_{l=1}^{m}x_l=1$ and $0 \leq x_l \leq 1$, $l=1,2...,m$.
Third, it is known that the MLD algorithm can have different performances on different types of benchmark datasets. 
Thus, we consider $k$ benchmark datasets with $k$ different types $DS_1, \ldots, DS_k$ for being used as a categorical factor 
\begin{equation}
    z_4 =
    \begin{cases}
     1 & \text{dataset for use is $DS_{1}$},\\
     \vdots \\
     k & \text{dataset for use is $DS_{k}$}. \nonumber\\
    \end{cases}
\end{equation}

For the response variables as the performance metrics of MLD algorithm, we consider both detection accuracy $y_{1}$ and prediction accuracy $y_{2}$.
Prediction accuracy is the classification accuracy. 
For the MLD algorithm, detection accuracy is an important metric to assess mislabeled data detection methods.

\subsection{Design Construction}\label{design.construction}

Note that the factors we consider 
contain both continuous factors $z_{1}, z_{2}, z_{3}$ without constraint, categorical factor $z_{4}$, and $m$ continuous factors $x_{1}, x_{2}, \ldots, x_{m}$ with constraint.
For such a complicated high-dimensional constraint space, it is not practical to choose a large set of random design points (i.e., combination of factors) from the search space to investigate the quality of the MLD algorithm due to the limited computational resources.
However, finding a set of representative design points in such a complicated space is not trivial.
To address this issue, we consider a space-filling design in the constraint space by leveraging the maximum projection from \cite{joseph2015maximum} under the setting of the constraint space. 

First, we construct a space-filling design for $x_{1}, x_{2}, \ldots, x_{m}$ with constraint $\sum_{l=1}^{m}x_l=1$ and $0 \leq x_l \leq 1$, $l=1,2 \ldots, m$.
Denote the design with $N$ points as $\bm X_{D}= (x_{il})_{N \times m}$, $i=1,\ldots, N; l=1,2, \ldots, m$. 
To pursue the space-filling on a constraint space, we consider finding the design by
\begin{align}\label{eq: constraint-Maxpro}
    \bm X_{D} &= \arg \min_{\bm X} \sum_{i=1}^{N-1}\sum_{j=i+1}^{N}\frac{1}{\prod_{l=1}^{m}(x_{il}-x_{jl})^2} \nonumber \\
              & \quad s. t. \quad  \sum_{l=1}^{m}x_{il}=1,  \sum_{l=1}^{m}x_{jl}=1;  \\
              & \quad   \quad   0 \leq x_{il} \leq 1,  0 \leq x_{jl} \leq 1, \forall i,j,m. \nonumber
\end{align}

Note that the above constraint $\sum_{l=1}^{m} x_l=1, 0 \leq x_l \leq 1, l=1,2...,m $ defines a subspace where all dimensions have the equal importance.
It implies that one could consider the modified maximin criterion \citep{joseph2015maximum} 
\begin{align}\label{eq: maxmin}
f(\bm X|\bm \theta)=\sum_{i=1}^{N-1}\sum_{j=i+1}^{N}\left( \frac{1}{\sum_{l=1}^{m}\theta_{l}(x_{il}- x_{jl})^2} \right )^{p/2}
\end{align}
for searching the optimal design with $\bm \theta = (\theta_{1}, \ldots, \theta_{m})'$ following a uniform distribution. 
The following theorem justifies that our adopted criterion is the modified maximin criterion in expectation.

\begin{theorem}\label{thr:valid}
Suppose $\bm X=(x_{il})_{N\times m}$ is a design on the subspace $S_{b} = \{\bm x \in \mathcal{R}^{m}: \sum_{l=1}^{m} x_l=1, 0 \leq x_l \leq 1, l=1,2...,m\}$. 
For $f(\bm X|\bm \theta)$ in \eqref{eq: maxmin} with $p=2m$ and $\bm \theta$ following a uniform distribution in the region $H = \{\bm \theta:  0 \leq \theta_{l} \leq 1, \sum_{l=1}^{m} \theta_{l}=1 \}$. Then 
\begin{align*}
E_{\bm \theta}(f(\bm X| \bm \theta))&=\int_{H} f(\bm X|\bm \theta)d p(\theta)\\ &= C \sum_{i=1}^{N-1} \sum_{j=i+1}^{N}   \frac{1}{\prod_{l=1}^{m}(x_{il}-x_{jl})^{2m}},
\end{align*}
where $C$ is a constant.
\end{theorem}

To implement the optimization in \eqref{eq: constraint-Maxpro} in search of optimal design, it is a nonlinear optimization with the number of parameters as $Nm$. Thus, the objective function can be very complex and the optimization process can be hard to execute. 
To overcome this difficulty, we adopt the coordinate exchange from discrete optimization and coordinate descent from continuous optimization 
to solve the optimization efficiently. Algorithm 1 summarizes the procedure. 
Specifically, we randomly choose $N$ runs from the candidate set $\mathcal{A}$ as the initial design.
Then one can find one run in the initial design that has the maximum $\sum_{i\neq j}\prod_{l=1}^{m}\frac{1}{(x_{il}-x_{jl})^2}$.
Replace this run with one run in the candidate set. 
If the criterion is reduced, we can further optimize the criterion by traditional constraint optimization; 
If not, replace this run with another run in the candidate set until one run can reduce the criterion.


\begin{algorithm}
\caption{Optimization Procedure}
\label{alg:opt}

\textbf{Input}: Candidate set $\mathcal{A}$, the number of design runs $N$, the number of redundant iterations $t$, the objective function $f(\cdot)$\\
\textbf{Output}: $X$ 
\begin{algorithmic}[1] 
\STATE Let $t=0$. Randomly choose $N$ runs from the candidate set $\mathcal{A}$ as the initial design $X$. Initialize $X_{new}, \quad s. t. \quad |f(X)-f(X_{new})|>\epsilon$, 
\WHILE{$t \leq 10000$}
\STATE Replace one row  of $X$ with one run $x_{a}$ from $\mathcal{A}$, denote the new design as $X_{new}$.
\IF {$f(X)>f(X_{new})$}
\STATE Constraint continuous optimization to replace $x_{a}$ with $x_{opt}$, update $X_{new}$.
\IF {$(|f(X)-f(X_{new})|<\epsilon)$}
\STATE $X=X_{new}$; algorithm converges and break the loop.
\ELSE
\STATE $X=X_{new}$.
\ENDIF
\ELSE
\STATE $t=t+1$
\ENDIF
\ENDWHILE
\STATE \textbf{return} $X$
\end{algorithmic}
\end{algorithm}

\begin{figure}[h]
\begin{center}
\includegraphics[width=0.95\textwidth]{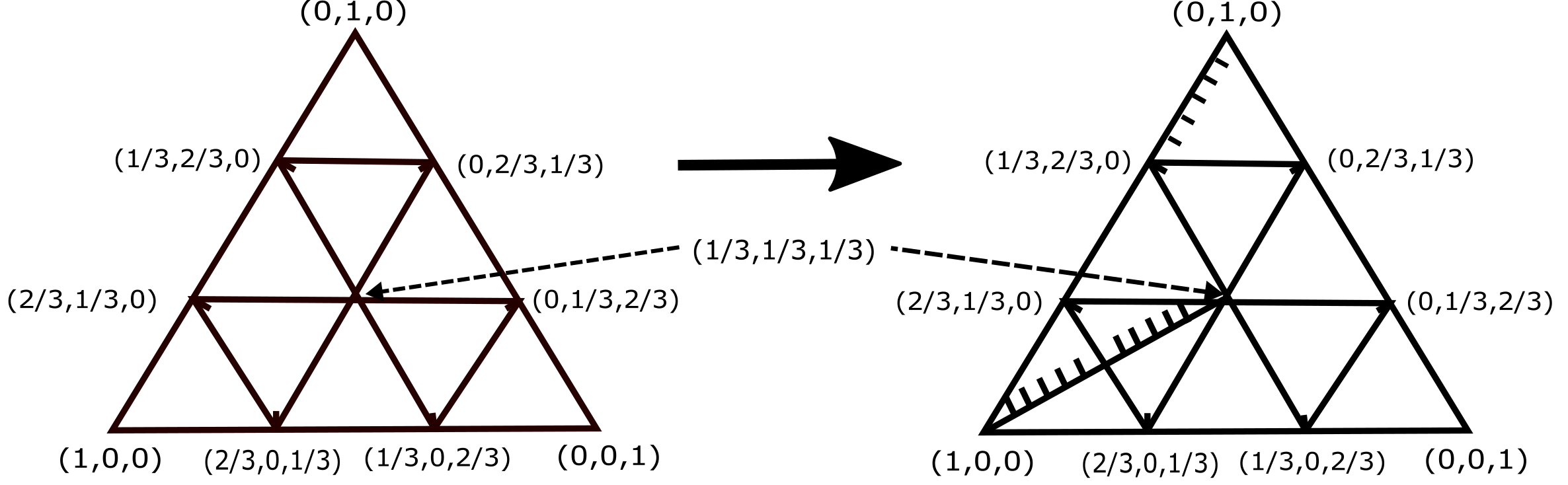}
\caption{3D simplex centroid design to construct candidate set.}\label{fig:candidate_set}
\end{center}
\end{figure}

Note that Algorithm 1 requires a good candidate set in the constraint space. A proper candidate set should have points distributed uniformly on the subspace. A naive approach is to generate a space-filling design in the hypercube $S = \{\bm x \in \mathcal{R}^{m}: 0 \leq x_l \leq 1, l=1,2..., m\}$ and project the points into hyperplane $S_{b}=\{ \bm x \in \mathcal{R}^{m}: 0 \leq x_l \leq 1,\sum_{l=1}^{m} x_l=1 \}$. 
However, such an approach can have a high rejection rate, especially in the high-dimensional setting since not all projected points will be in the constraint space $S_{b}$.
To address this challenge, we propose an algebraic construction of the candidate set $\mathcal{A}$ based on the simplex centroid design \citep{scheffe1963simplex}. 

For $m$ dimensions, there are $2m-1 \choose m$ possible simplex centroid design points.
The candidate set $\mathcal{A}$ consists of all points in the simplex centroid design and points on the segment between the two points of the simplex centroid design.
Figure \ref{fig:candidate_set} illustrates an example of how we deploy the original candidate set to construct more samples for a candidate set of three dimensions.
Note that such a construction of the candidate set is naturally space-filling based on the simplex centroid design, is also very flexible on the run size and computationally fast. 

With the space-filling design for the continuous factors $x_{1}, \ldots, x_{m}$ with linear constraint, we then can construct a cross array design between $x_{1}, \ldots, x_{m}$, other continuous factors, and categorical factors as the complete design. 
Here the Latin hypercube design \citep{park1994optimal} is used to design construction for other continuous factors. 
In the next section, we will detail using the proposed design and collected responses to build a surrogate model for studying how hyperparameters of AI algorithms and data quality factors impact the quality of the MLD algorithm.

\subsection{Surrogate Modeling}\label{sec:GPmodel}
With a large number of factors in our investigation, a quadratic linear regression model can contain too many model parameters while a first-order linear regression model can be too simple to emulate the intricate relationship between the factors and metrics. 
Therefore, a proper surrogate model is needed. Here we consider the Gaussian process to be used for surrogate modeling.
Note that our DoE contains both continuous factors and discrete factors. To address this challenge, 
we adopt additive Gaussian process (AGP) used in \citep{deng2017additive} as our surrogate.

With loss of generality, suppose that the response is $\tilde{y}$, and the corresponding covariates are $\tilde{\bm w} =(\tilde{\bm x}, \tilde{\bm z})$, 
where $\tilde{\bm x}$ is a continuous covariate vector, and $\tilde{\bm z} = (\tilde{z}_{1}, \ldots, \tilde{z}_{\tilde{q}})$ is a binary covariate vector with $\tilde{q}$ dimensions. 
Then we consider AGP for modeling the response $\tilde{y}$ as
\begin{align}
    \tilde{y}(\tilde{\bm w}) = \mu+ \sum_{h=1}^{\tilde{q}} G_{h}(\tilde{z}_{h},\tilde{\bm x}),
\end{align}
where $G_{h}$'s are independent Gaussian processes with mean zero and the covariance function $\Sigma_{h}$'s. 
Here $\mu$ is the overall mean and we set $\mu=0$ for simplicity.
The covariance function of $G_{h}$ between a pair of observations $\tilde{y}_{i}(\tilde{\bm w}_{i})$ and $\tilde{y}_{j}(\tilde{\bm w}_{j})$ is defined as $\Sigma_{h}(\tilde{\bm w}_{i}, \tilde{\bm w}_{j}) = \tau^2 \psi_{h}(\tilde{\bm x}_{i}, \tilde{\bm x}_{j}) \phi_{h}(\tilde{z}_{hi},\tilde{z}_{hj})$ where
\begin{align*}
    \psi_{h}(\tilde{\bm x}_{i}, \tilde{\bm x}_{i}) & = exp(-\frac{\|\tilde{\bm x}_{i} - \tilde{\bm x}_{j}\|^2}{\vartheta_{h}}) + \eta \mathbbm{1}(i = j), \\
    \phi_{h}(\tilde{z}_{hi},\tilde{z}_{hj}) & = exp( \mathbbm{1}(\tilde{z}_{hi},\tilde{z}_{hj}) \log \rho_{h}),
\end{align*}
where $\mathbbm{1}(\cdot)$ is an indicator function. Here $\vartheta_{h} \ge 0$ and $0 \le \rho_{h} \le 1$. The parameter $\eta \ge 0$ is a nugget effect.  
Thus, the overall covariance function $\Sigma(\cdot,\cdot)=\sum_{h=1}^{\tilde{q}}\Sigma_{h}(\cdot,\cdot)$.

When using AGP for our design factors in Section \ref{factor.response}, the continuous covariate vector $\tilde{\bm x}$ contains $x_{1}, \ldots, x_{m}$ and $z_{1}, \ldots, z_{3}$. Recall that the categorical factor $z_4$ has $k$ levels as $k$ variaous types of datasets. 
While we consider the binary covariate vector $\tilde{\bm z} = (\tilde{z}_{1}, \ldots, \tilde{z}_{\tilde{q}})$ for AGP with $\tilde{z}_{h}$ to be
\begin{align*}
\tilde{z}_{h} = \left\{
  \begin{array}{ll}
    0, & \hbox{$DS_{1}$ is used;} \\
    1, & \hbox{$DS_{h+1}$ is used.}
  \end{array}
\right.
\end{align*}
where $h=1,2,...,k-1$
Now assume that the collected data are  $(\tilde{\bm y}, \tilde{\bm W})$, where $\tilde{\bm y}$ is the $N$ dimensional vector of responses and $\tilde{\bm W}$ the corresponding covariate matrix. 
Given a new design point $\tilde{w}_{*}$, its corresponding response $\tilde{y}_{*}$ can have 
\begin{center}
$\begin{pmatrix}
\tilde{\bm y}\\
\tilde{y}_{*}
\end{pmatrix}
\sim \mathcal{N}_{N+1}
\begin{pmatrix}
\begin{pmatrix}
   \Vec{0}_{N}\\
   0
\end{pmatrix}
,
 \begin{pmatrix}
     \Sigma_{\tilde{\bm y},\tilde{\bm y}} & \Sigma_{\tilde{\bm y}, \tilde{y}_{*}}\\
     \Sigma_{\tilde{y}_{*}, \tilde{\bm y}} & \Sigma_{\tilde{y}_{*}, \tilde{y}_{*}}
 \end{pmatrix}
\end{pmatrix}$,
\end{center}
where $\Sigma_{\tilde{\bm y},\tilde{\bm y}}$ is the covaraince matrix for $\tilde{\bm y}$,  $\Sigma_{\tilde{\bm y}, \tilde{y}_{*}}$ is the covariance vector between $\tilde{\bm y}$ and $\tilde{y}_{*}$ ($\Sigma_{ \tilde{y}_{*}, \tilde{\bm y}}$ is the covariance vector between $\tilde{y}_{*}$ and $\tilde{\bm y}$ ), and $\Sigma_{\tilde{y}_{*}, \tilde{y}_{*}}$ is variance of $\tilde{y}_{*}$.
Therefore, it is obvious that the conditional distribution of $\tilde{y}_{*}$ given $\tilde{\bm y}$ can be obtained for prediction and uncertainty quantification as
$\tilde{y}_{*}|\tilde{\bm y} \sim \mathcal{N}( \mu_{*}, \sigma^{2}_{*})$ where
\begin{align}\label{eq:mean}
    \mu_{*} & = \Sigma_{\tilde{y}_{*}, \tilde{\bm y}} \Sigma_{\tilde{\bm y},\tilde{\bm y}}^{-1} \tilde{\bm y},\\
     \sigma^{2}_{*} & =\Sigma_{\tilde{y}_{*}, \tilde{y}_{*}} - \Sigma_{ \tilde{y}_{*},\tilde{\bm y}}  \Sigma_{\tilde{\bm y},\tilde{\bm y}}^{-1}  \Sigma_{\tilde{\bm y}, \tilde{y}_{*}}. \nonumber
\end{align}
For simplicity, define $\bm \rho=(\rho_1,...,\rho_{\tilde{q}}), \bm \vartheta=(\vartheta_1,...,\vartheta_{\tilde{q}})$. To estimate unknown parameters $\bm \beta=(\bm \rho, \bm \vartheta, \eta)$ and $\tau^{2}$, calculate $\hat{\tau}^2=\tilde{\bm y}^{T}\Sigma_{\tilde{\bm y},\tilde{\bm y}}^{-1}\tilde{\bm y}/N$ firstly and then minimize the negative log-likelihood function given $\hat{\tau}^2$ as
\begin{align*}
    \hat{\bm \beta} = \arg \min_{\bm \beta}  \left \{ \log(\det(\Sigma_{\tilde{\bm y},\tilde{\bm y}})) + Nlog (\tilde{\bm y}^{T}\Sigma_{\tilde{\bm y},\tilde{\bm y}}^{-1} \tilde{\bm y}  ) \right \}.
\end{align*}

The optimization can be solved by the derivative-based method. The estimation process is conducted in a iterative manner. See more details in the Appendix.

\section{Numerical Experiments}\label{sec:model.analysis}

\subsection{Design Validation and Data Collection}
The validity of the proposed space-filling design 
for the class proportions (i.e., $x_{1}, \ldots x_{m}$) is examined by comparing it with a benchmark method, Kennard and Stone algorithm \citep{kennard1969computer}. Here we consider two performance measures ``Coverage" and ``Maxmin" to evaluate the space-filling property, denote as $PM_{1}$ and $PM_{2}$ as
\begin{align*}
PM_{1}&= \frac{1}{\bar{d}}\sum_{i=1}^{N}(d_i-\Bar{d_i})^2, \\
PM_{2}&= Max(d_i),
\end{align*}
where $d_i= \underset{j\neq i}{min}\|x_i-x_j\|$ and $\Bar{d}=\underset{all ~i}{mean}(d_i)$; $i,j=1,2,...,N$. These two metrics are the bigger the better. The comparison results are reported in Table \ref{tab:metric_compare} with $m=10$. 
It is seen that the proposed design has consistent advantage over the benchmark method (K-S) in terms of $Coverage$. 
In terms of $Maxmin$, 
the proposed design is better than the K-S method when the design runs is relatively large. 

We conduct data analysis based on additive Gaussian process modeling with respect to prediction accuracy $y_{1}$ and detection accuracy $y_{2}$. 
The range of weights ratio $z_1$ is from 1/500 to 500. 
The range of $z_2$ (i.e., the value of $\alpha$) is from $1$ to $3$. 
The range of percentage of mislabeled data $z_3$ is from $10\%$ to $50\%$.
Here we use {\bf MNIST} \citep{deng2012mnist} and {\bf FashionMNIST}  \citep{xiao2017fashion} as two types of benchmark datasets. 
The input images of these two datasets are both $28\times28$, and both have $10$ classes. Thus, the proportions of classes are denoted $x_1,...,x_{10}$, such that $0 \le x_l \le 1$, $l=1,...,10$. For both datasets, there are 60,000 observations of the training set and 10,000 observations of the test set. 
According to the design we obtained, we run $5$ replicates for each setting out of $2,000$ in total (i.e., $10,000$ design runs in total). 
The detection accuracy and prediction accuracy are collected for further analysis. 
See the finalized design table in the appendix.

\begin{table}
\begin{center}
\caption{Compare our approach based on ``Maxpro'' criterion of a subspace with Kennard and Stone (K-S) Algorithm.}\label{tab:metric_compare}
\vspace{1.5ex}
\begin{tabular}{lrrrr}\hline\hline
Runs ($N$) & Method & $PM_{1}$ &  $PM_{2}$
\\\hline
50 & Proposed Design & \textbf{0.1509}   & 0.2962 \\
 & K-S  & 0.0287 & \textbf{0.3782}\\ \hline
 100 & Proposed Design & \textbf{0.2149} & \textbf{0.4490}\\
 & K-S & 0.0440 & 0.3403\\ \hline
 150 & Proposed Design & \textbf{0.1852} & \textbf{0.3926} \\
 & K-S & 0.1579 & 0.3768\\ \hline
 200 & Proposed Design & \textbf{0.2099} & \textbf{0.3795}\\
 & K-S &  0.2022 & 0.3672 
\\ \hline\hline
\end{tabular}
\end{center}
\end{table}

\subsection{Visualization of Findings} \label{sec: data-visualization}

Here, we visualize the data from the experimental results with respect to detection accuracy $y_{2}$ and prediction accuracy $y_{1}$. Figure \ref{fig:detec_3c} displays the boxplots of detection accuracy $y_{2}$ with respect to the series of specific weights ratio $z_{1}$, the threshold of deviation $z_{2}$, and the proportion of mislabeled training data $z_{3}$, given dataset MNIST and FashionMNIST. See the first two columns, displaying the relationships between hyperparameters of the algorithm, $z_{2}$, $z_{3}$  and detection accuracy $y_{2}$. Given dataset MNIST, the mean of $y_{2}$ remains at a low level when $z_{1}$ is smaller than $1$ and dramatically increases above $0.9$ and fluctuates around $0.93$ when $z_{1}$ is greater than $1$. The maximum mean is $0.9542$ and the corresponding $z_{1}$ is $154.92$. The trend of mean of $y_{2}$ versus $z_{1}$ for dataset FashionMNIST is quite similar. $y_{2}$ remains under $0.76$ when $z_{1}$ is under $1$, then fluctuates around $0.91$ as $z_{1}$ goes above $1$, and achieves the maximum point $0.9386$ with $z_{1}$ equals $154.92$.
Move to mid column of the figure, for the dataset MNIST, the pattern of the mean of $y_{2}$ with respect to corresponding $z_{2}$ is not clear. The mean of $y_{2}$ goes up and down and two
peaks $y_{2}=0.9542,0.9515$ are achieved with $z_{2}=1.7732,2.2797$. Dataset FashionMNIST gives a consistent pattern as the dataset MNIST.

\begin{figure}
\begin{center}
\includegraphics[width=0.75\textwidth]{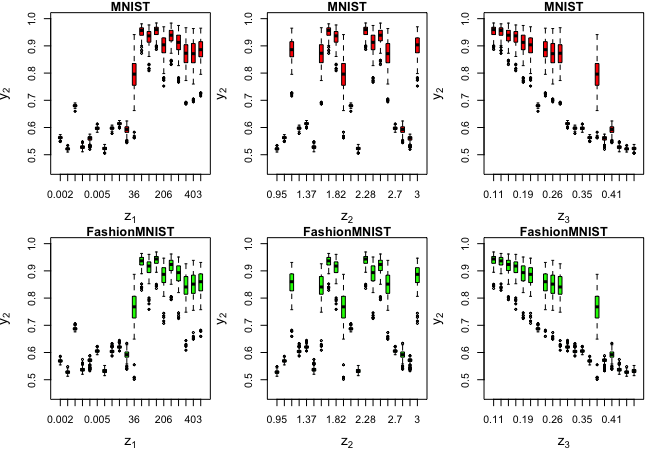}
\caption{Detection accuracy $y_{2}$ versus weights ratio $z_{1}$ (Left), the deviation of threshold $z_{2}$ (Mid), and proportion of mislabeled training data $z_{3}$ (Right) for MNIST (Upper) and FashionMNIST (Lower) datasets.}\label{fig:detec_3c}
\end{center}
\end{figure}

Now move to data quality factors. In general, mean of $y_{2}$ decreases as the proportion of mislabeled training data $z_{3}$ increases in almost a straight line for both MNIST and FashionMNIST datasets. It indicates that the percentage of mislabeled training data directly and linearly affect the algorithm's ability to distinguish mislabeled data from normal data.
The detection accuracy of two datasets is quite the same. It demonstrates the detection algorithm has similar detection accuracy on two different types of datasets. Now we introduce the variance of proportions of classes for one design run, $z_{5i}=var(x_{1i},x_{2i},...,x_{10i})$, for $i=1,2,...,10000$.
The greater the variance of proportions of classes, the more class imbalance in the training data.
See the second column of Figure \ref{fig:pred_detect_var}, for both MNIST and FashionMNIST datasets, the mean of $y_{2}$ is slightly decreasing as the variance of proportions of classes $z_{5}$ is increasing. However, the range of detection accuracy $y_{2}$ for a given $z_{5}$ is large and the minimums are quite the same.
These evidences indicate that class imbalance does negatively affect the detection accuracy but it is not a substantial factor. 

\begin{figure}
\begin{center}
\includegraphics[width=0.75\textwidth]{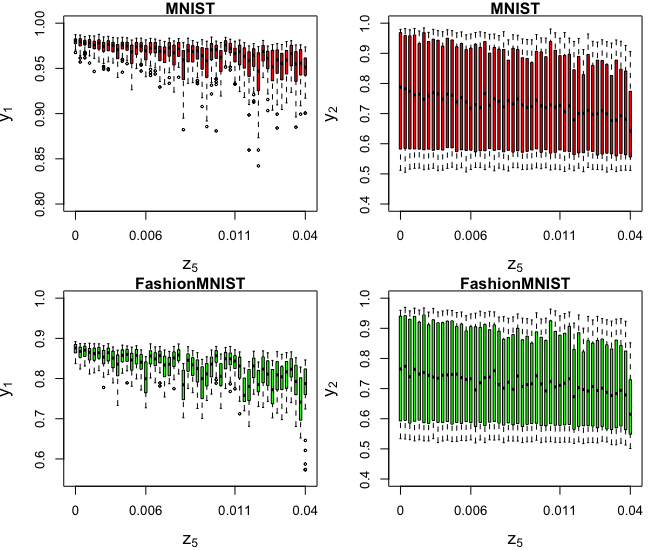}
\caption{Prediction accuracy $y_{1}$ (Right) and detection accuracy $y_{2}$ (Left) versus variance of proportions of classes $z_{5}$ for MNIST (Upper) and FashionMNIST (Lower) datasets.}\label{fig:pred_detect_var}
\end{center}
\end{figure}

We can see there are several narrow boxes (low variances) with quite low average $y_{2}$ values. It indicates that detection accuracy can stably remain in a low level. However, higher average values do not correlate with a higher variance in detection accuracy. Regardless of all points with the lower average and variance of performance, the greater the average the smaller the variance of detection accuracy. This phenomenon reveals that if the layouts of factors enable the algorithm to produce a decent average performance (approximately greater than 0.7), the algorithm is more and more stable as the performance increases.

Now consider prediction accuracy $y_1$ as the performance measure. The patterns displayed in Figure \ref{fig:pred_3c} is similar to the patterns in terms of detection accuracy $y_2$. But there are several distinctions between them.
The patterns that $y_{1}$ versus $z_{1}$, $z_{2}$, and $z_3$ for two datasets have almost consistent trends. The curves go up and down and achieve the maximums at the consistent certain points on $X-axis$. However, the average prediction accuracy of MNIST data is greater than the average of FashionMNIST data.
\begin{figure}
\begin{center}
\includegraphics[width=0.75\textwidth]{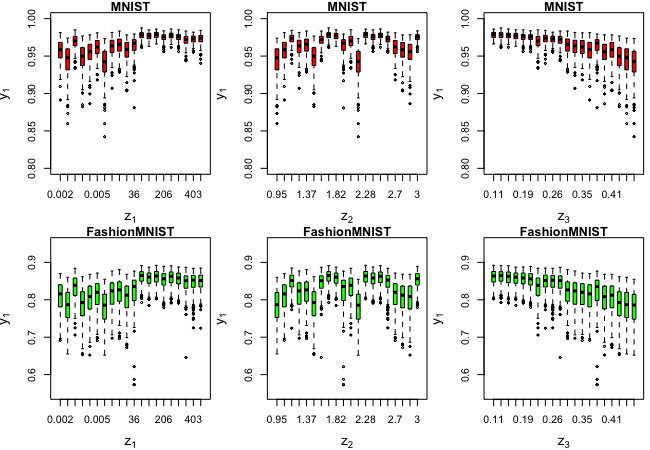}
\caption{Prediction accuracy $y_{1}$ versus weights ratio $z_{1}$ (Left), the deviation of threshold $z_{2}$ (Mid), and proportion of mislabeled training data $z_{3}$ (Right) for MNIST (Upper) and FashionMNIST (Lower) datasets.}\label{fig:pred_3c}
\end{center}
\end{figure}
It indicates that the data type $z_{4}$ has a profound effect on the prediction accuracy $y_1$. When two convolutional neural network (CNN) models with the consistent structures are well trained both on the clean training datasets of MNIST and FashionMNIST, the CNN model trained on MNIST dataset has a distinctly higher prediction than the model trained on FashionMNIST. In comparison to the relationship between $z_{4}$ and detection accuracy $y_{2}$, the detection algorithm may rely on the differences between mislabeled and correctly labeled data for given classes rather than primary information in training set, whereas CNN is more dependent on the primary information. 
See the first column of Figure \ref{fig:pred_detect_var},
the trends of curves that variance of proportions of classes $z_5$ versus the mean of $y_1$ for both MNIST and FashionMNIST are quite consistent. Low $z_5$ (less imbalance) always occurs with a high prediction accuracy $y_1$. 

The length of boxes is longer when the average prediction accuracy $y_1$ is greater. It shows that the high variance of prediction accuracy leads to a low mean of prediction accuracy. For prediction accuracy, we do not see short boxes that appears in the case of detection accuracy $y_{2}$, where if the average of $y_{2}$ is below a threshold, the variance $y_{2}$ is quite low. It indicates that both the average and variance of performance can represent the quality of CNN's application as they have a monotonic relationship.

\subsection{Modeling Results}\label{sec: modleing-results}

To justify the proposed surrogate model,
we partition the whole experimental data into the training set ($80\%$) and the test set ($20\%$).
The AGP is used as the surrogate model for the responses $y_{1}$ (i.e., prediction accuracy) and $y_{2}$ (i.e., detection accuracy), respectively.
For comparison, the linear regression model is used as the benchmark model to fit the training data and make prediction on the test set.
To measure the accuracy of the AGP model, we choose the mean square error (MSE) in the test set as the metric to assess the goodness-of-fit. 
For the response $y_{1}$ (i.e., prediction accuracy), the MSE of the AGP is $8.7147\times 10^{-5}$ and the MSE of the regression is $0.0028$.
For the response $y_{2}$ (i.e., detection accuracy), the MSE of the AGP is $6.1186 \times 10^{-5}$ and the MSE of the regression is $0.0006$.
Clearly, the proposed AGP models perform much better than the benchmark method. 
Furthermore, Figure \ref{fig:resid.dist} reports the  distributions of residuals of AGP models with respect to $y_{1}$ and $y_{2}$. 
They are much narrower than the corresponding regression models. 

\begin{figure}
\begin{center}
\includegraphics[width=0.65\textwidth]{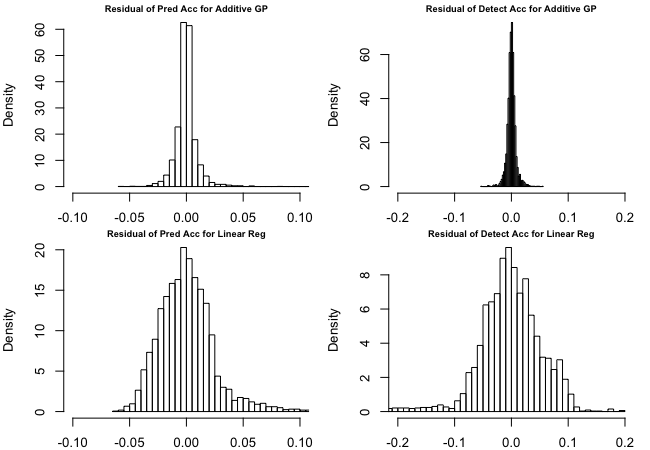}
\caption{Residual distributions of additive Gaussian model (AGP) (Upper) and linear regression model (Linear Reg) (Lower) with respect to prediction accuracy $y_{1}$ (Pred Acc) and detection accuracy $y_{2}$ (Detect Acc).}\label{fig:resid.dist}
\end{center}
\end{figure}

\subsection{Summary of Findings}\label{finding.summary}

First, it is not surprising that the proportions of mislabeled data significantly affects the performance of detection accuracy. Moreover, changed in class imbalance can also diminish the detection accuracy. 
However, not all data quality factors have profound effect on the detection accuracy.
In our study with two types of benchmark datasets (MNIST and FashionMNIST), the detection accuracy is robust against the type of data used in the proposed method. 
We observed that the optimal setting of two hyperparameters, weights ratio $z_{1}$ and threshold $z_{2}$ are quite similar when the data types (datasets) are different. It indicates that the MLD algorithm has a similar capacity to derive the differences between malicious items and normal items when the datasets are not the same and the hyperparameters of the detection algorithm do not have an interaction effect with data type. 
Second, the MLD algorithm turns conservative (detection accuracy in low variance) when its capacity to detect the mislabeled data is poor. 
When the detection accuracy is relatively high, in our cases, roughly larger than $0.7$, the stability and the average performance of the detection algorithm are positively associated. 
Third, the detection accuracy and the classification accuracy can be of different characteristics. 
A combination of both metrics can help distinguish if the algorithm is affected by the primary information in the dataset (data type factor).

\section{Related Work}\label{sec:related.literature}

In the literature, \cite{rushby1988quality} discussed the idea of AI quality measurement and assurance several decades ago. 
AI reliability and robustness (current concepts in AI quality) become increasingly popular in both academic research and industrial applications \citep{virani2020justification,dietterich2017steps,russell2015research}. 
For example, \cite{lian2021robustness} propose a design-of-experimental framework for investigating AI robustness as it relates to class imbalance issue and distribution shift of classes between training and test data. 
Several works use experimental designs to tune hyperparameters of AI algorithms \citep{packianather2000optimizing, staelin2003parameter,balestrassi2009design, mutny2020experimental}.
The evaluation of the AI quality needs a meticulous design of experiment since the size of search space can be very large \citep{bell2022modeling}.
The commonly used design for the search space is to have the space-filling property \citep{joseph2016space}.
Among various space-filling designs, the Latin hypercube design and its variants have received a lot of attention \citep{jung2011artificial}.
For example, \cite{joseph2015maximum} propose a so-called MaxPro design to maximize the design's space-filling capacity on all subspaces.
Note that the design for the continuous factors of class proportions has a linear constraint, which is a mixture design in the literature \citep{cornell2011experiments}. 
Typical construction of mixture designs \citep{gomes2018space} is not for high dimensions due to the computational complexity.
Quadratic and cubic linear models \citep{PiepelCornell1994} are used to analyze mixture designs in low dimensions. 
Gaussian process (GP) models are effective as an efficient surrogate model for a complex system \citep{bernardo1998regression} in high dimensional space. Note that GP is usually valid for continuous variables \citep{cardelli2019robustness,ling2016gaussian}. When the data contains the categorical input variables, one may use one-hot encoding to transfer a categorical variable to a continuous variable \citep{garrido2020dealing}.  Additive Gaussian process \citep{deng2017additive} has been proposed to handle both qualitative and quantitative factors.

\section{Discussion}\label{sec:conclusion_future_direction} In this work, we establish a Do-AIQ framework to comprehensively investigate how data quality factors (e.g., data type, percentage of mislabeled data, class imbalance in training data) and hyperparameters of algorithms affect the quality assurance of the AI algorithms.

There could be several limitations of the current work.
First, we choose a mislabeled data detection algorithm to describe the proposed framework. 
However, the Do-AIQ is not limited to this particular MLD algorithm. 
It can be extended to other AI algorithms whose quality assurance is affected by various data quality factors and their internal structure.
%
Second, the number of classes in this  study is around ten, not as large as hundreds in some benchmark data. 
For data with large clasess such as the CIFAR100 dataset with 100 classes,  
the proposed design construct algorithm may encounter certain computational difficulties.

In the future work, one can consider sequential design based on additive Gaussian process to handle the sensitivity concern and limitation of design runs. 
Additionally, it is also appealing to extend the Do-AIQ framework for multiple types of datasets with the different numbers of classes. 
Another future direction for the proposed Do-AIQ framework is to ensure fairness in the design of experiments.

\bibliographystyle{apalike}
\bibliography{aaai23}

\clearpage
\appendix

\section{Proof of Theorem $1$}

\begin{proof}
Let $p=2m$ and $\bm \theta$ follow a uniform distribution in the region $H = \{\bm \theta:  0 \leq \theta_{l} \leq 1, \sum_{l=1}^{m} \theta_{l}=1 \}$. Then $d p(\bm \theta)=\frac{(m-1)\,!}{\sqrt{m}} d \bm \theta$. 
\begin{align*}
    E_{\bm \theta}(f(\bm X| \bm \theta)) 
    & =\int_{H} f(\bm X|\bm \theta)d p(\bm \theta)\\ 
    & = \int_{H} f(\bm X|\bm \theta) \frac{(m-1)\,!}{\sqrt{m}} d \bm \theta\\ 
    &= \int_{S_{m-1}} \sqrt{m} f(\bm X|\theta_{1},...,\theta_{m-1}, 1-\sum_{l=1}^{m-1}\theta_{l})\\ 
    & \quad \times \frac{(m-1)\,!}{\sqrt{m}} d\theta_{1}...d\theta_{m-1}\\ 
    &= (m-1)\,!  \int_{S_{m-1}} f(\bm X|\theta_{1},...,\theta_{m-1}, 1-\sum_{l=1}^{m-1}\theta_{l})\\ 
    & \quad  d\theta_{1}...d\theta_{m-1}\\ 
    =& (m-1)\,! \sum_{i=1}^{N-1}\sum_{j=i+1}^{N} \int_{S_{m-1}} f_{(i,j)}(\bm X|\theta_{1},...,\theta_{m-1})\\ 
    & \quad d\theta_{1}...d\theta_{m-1},
    \end{align*}
    where $S_{m-1}= \{\bm \theta:  0 \leq \theta_{1},\theta_{2},...,\theta_{m-1} \leq 1, \sum_{l=1}^{m-1} \theta_{l} \leq 1 \}$, $f_{(i,j)}(\bm X|\theta_{1},...,\theta_{m-1})=  \left( \sum_{l=1}^{m-1}\theta_{l}d_{(i,j)l} + (1-\sum_{l=1}^{m-1}\theta_{l})d_{(i,j)m} \right) ^{-m}$, $d_{(i,j)l}=(x_{il}-x_{jl})^2$. Suppose
    \begin{align*}
        Q_{m}(m,a)& = \int_{S_{m-1}} \left( \sum_{l=1}^{m-1}\theta_{l}d_l + (1-\sum_{l=1}^{m-1}\theta_{l})a \right )^{-m}\\
        & \quad d\theta_{1}...d\theta_{m-1}. 
    \end{align*}
    For $a\neq d_{m-1}$, $Q_{m}(m,a)=\frac{1}{(m-1)(a-d_{m-1})}(Q_{m-1}(m-1,d_{m-1})-Q_{m-1}(m-1,a))$. It is easy to see that $Q_2(2,a)=1/(d_{1}a)$. Therefore, $Q_{m}(m,a)=1/\{(m-1)\,!d_{1}...d_{m-1}a\}$, then $Q_{m}(m,d_{m})=1/\{(m-1)\,!d_{1}...d_{m-1}d_m\}$ \citep{joseph2015maximum}.
\end{proof}


\newpage
\section{Validation of Covariance Matrix of Additive Gaussian Process}

\begin{lamma}Schur Product Theorem.\label{schur:prod}
    If $\bm D$ is an $n \times n$ positive semidefinite matrix with no diagonal entry equal to zero and $\bm E$ is and $n \times n$ positive definite matrix, then $\bm D \circ \bm E$ is positive definite. If both $\bm D$ and $\bm E$ are positive definite, then $\bm D \circ \bm E$ is positive definite as well \citep{horn2012matrix}.
\end{lamma}

\begin{proposition}
For the collected data $(\tilde{\bm y}, \tilde{\bm W}=(\tilde{\bm X},\tilde{\bm Z}))$, where $\tilde{\bm X}$ is the $N \times \tilde{p}$ continuous covariates matrix and $\tilde{\bm Z}$ is the $N \times \tilde{q}$ binary covariates matrix. Then the covariance matrix of $\tilde{\bm y}$ is a positive definite matrix.
\end{proposition}

\begin{proof}
    The covariance matrix for $\tilde{\bm y}$ is 
    \begin{align*}
        \Sigma_{\tilde{\bm y},\tilde{\bm y}} &=  \sum_{h=1}^{\tilde{q}}(\Sigma_{h}(\tilde{\bm w}_{i}, \tilde{\bm w}_{j}))_{N\times N} \\ 
        &= \sum_{h=1}^{\tilde{q}}(\tau^2 \psi_{h}(\tilde{\bm x}_{i}, \tilde{\bm x}_{j}) \phi_{h}(\tilde{z}_{hi},\tilde{z}_{hj}))_{N \times N}\\
        &= \sum_{h=1}^{\tilde{q}} \tau^2(\psi_{h}(\tilde{\bm x}_{i}, \tilde{\bm x}_{j}))_{N\times N} \circ (\phi_{h}(\tilde{z}_{hi},\tilde{z}_{hj}))_{N\times N}, 
    \end{align*}
    where $\circ$ is Schur product, $i,j=1,2,...,N$. Denote $(\psi_{h}(\tilde{\bm x}_{i}, \tilde{\bm x}_{j}))_{N\times N}$ as $\bm C_{h}$, $(\phi_{h}(\tilde{z}_{hi},\tilde{z}_{hj}))_{N\times N}$ as $\bm B_{h}$. Note that $\bm C_{h}$'s are positive definite matrices. $\bm B_h$ can be obtained by 
    \begin{align}
        \bm B_{h} = \tilde{\bm A}_h  \begin{pmatrix} 
        1 & \rho_h\\
        \rho_h & 1
        \end{pmatrix} \tilde{\bm A}_h^{T}, \nonumber
    \end{align}
    where $(\tilde{\bm A}_h)$ is a $\tilde{q}\times 2$ matrix, and \begin{align*}
    (\tilde{\bm A}_h)_{i,1} = \left\{
  \begin{array}{ll}
    0, & \hbox{if $\tilde{z}_{hi}=0$;} \\
    1, & \hbox{if $\tilde{z}_{hi}=1$.}
  \end{array} \right.
    (\tilde{\bm A}_h)_{i,2} = \left\{
  \begin{array}{ll}
    1, & \hbox{if $\tilde{z}_{hi}=0$;} \\
    0, & \hbox{if $\tilde{z}_{hi}=1$.}
  \end{array}\right.
\end{align*}
    Since $\begin{pmatrix} 
        1 & \rho_h\\
        \rho_h & 1
        \end{pmatrix}$ is positive definite, then $B_h$ is as least positive semidefinite with all diagonal entries equal to $1$. Followed by Lamma \ref{schur:prod}, $\bm B_h \circ \bm C_h$ is positive definite. Therefore, the covariance matrix $\Sigma_{\tilde{\bm y},\tilde{\bm y}}$ is positive definite. 
\end{proof}

\section{Derivatives of the Log-Likelihood Function}\label{log.derivative}
Given $\tilde{\bm y}$ and $\tilde{\bm W}$, we build the log-likelihood function

\begin{align}
    l(\bm \beta,\tau^2)= & -\frac{N}{2}log(2\pi)-\frac{N}{2}log(\tau^2)-\frac{1}{2}log(det(\Sigma_{\tilde{\bm y}, \tilde{\bm y}}))\nonumber\\
    & -\frac{1}{2\tau^2}\tilde{\bm y}^{T}\Sigma^{-1}_{\tilde{\bm y}, \tilde{\bm y}}\tilde{\bm y}. \nonumber
\end{align}
To maximize the log-likelihood function, firstly, we maximize it with respect to $\tau^2$ given fixed $\bm \beta$. Take derivative with respect to $\tau^2$
\begin{align*}
    0 \overset{set}{=}\frac{\partial l(\tau^2|\bm \beta)}{\partial \tau^2}=-\frac{N}{2\tau^2}+\frac{1}{2(\tau^2)^2}\tilde{\bm y}^{T}\Sigma^{-1}_{\tilde{\bm y}, \tilde{\bm y}}\tilde{\bm y}.
\end{align*}
Therefore, $\hat{\tau}^2=\tilde{\bm y}^{T}\Sigma_{\tilde{\bm y},\tilde{\bm y}}^{-1}\tilde{\bm y}/N$. Then we maximize the log-likelihood given $\tau^2=\hat{\tau}^2$. Take derivatives with respect to $\bm \beta=( \bm \rho,\bm \vartheta, \eta)$
\begin{align*}
    \frac{\partial \Sigma_{h}((\tilde{z}_{hi},\tilde{\bm x}_i),(\tilde{z}_{hj},\tilde{\bm x}_j))}{\partial \vartheta_h} &=\Sigma_{h}((\tilde{z}_{hi},\tilde{\bm x}_i),(\tilde{z}_{hj},\tilde{\bm x}_j))\\
    & \quad \times \frac{\|\tilde{\bm x}_i-\tilde{\bm x}_j\|}{\vartheta_{h}^2}, \\
    \frac{\partial \Sigma_{h}((\tilde{z}_{hi},\tilde{\bm x}_i),(\tilde{z}_{hj},\tilde{\bm x}_j))}{\partial \eta}
    &=
    \begin{cases}
    1~i=j\\
    0~i\neq j
    \end{cases},\\
    \frac{\partial \Sigma_{h}((\tilde{z}_{hi},\tilde{\bm x}_i),(\tilde{z}_{hj},\tilde{\bm x}_j))}{\partial \rho_{h}}
    &=
    \begin{cases}
    0~i=j\\
    \Sigma_{h}((\tilde{z}_{hi},\tilde{\bm x}_i),(\tilde{z}_{hj},\tilde{\bm x}_j))~i\neq j
    \end{cases}, \\
    \frac{\partial \Sigma^{-1}_{\tilde{\bm y}, \tilde{\bm y}}}{\partial \bm \beta} 
    &= -\Sigma^{-1}_{\tilde{\bm y}, \tilde{\bm y}}\frac{\partial \Sigma_{\tilde{\bm y}, \tilde{\bm y}}}{\partial \bm \beta} \Sigma^{-1}_{\tilde{\bm y}, \tilde{\bm y}},\\
    \frac{\partial log(det(\Sigma_{\tilde{\bm y}, \tilde{\bm y}}))}{\partial \bm \beta} 
    &= tr(\Sigma^{-1}_{\tilde{\bm y}, \tilde{\bm y}} \frac{\partial \Sigma_{\tilde{\bm y}, \tilde{\bm y}}}{\partial \bm \beta}), \\
    \frac{\partial l(\bm \beta|\hat{\tau}^2)}{\partial \bm \beta} 
    &= -\frac{1}{2}tr(\Sigma^{-1}_{\tilde{\bm y}, \tilde{\bm y}} \frac{\partial \Sigma_{\tilde{\bm y}, \tilde{\bm y}}}{\partial \bm \beta})\\
    & \quad +\frac{N}{2}\frac{\tilde{\bm y}^{T}\Sigma^{-1}_{\tilde{\bm y}, \tilde{\bm y}}\frac{\partial \Sigma_{\tilde{\bm y}, \tilde{\bm y}}}{\partial \bm \beta} \Sigma^{-1}_{\tilde{\bm y}, \tilde{\bm y}}\tilde{\bm y}}{\tilde{\bm y}^{T}\Sigma^{-1}_{\tilde{\bm y}, \tilde{\bm y}}\tilde{\bm y}},
\end{align*}
and obtain $\hat{\bm \beta}$ by minimizing the negative log-likelihood function given $\hat{\tau}^2$ through derivative based optimization, such as Broyden–Fletcher–Goldfarb–Shanno (BFGS) algorithm. Repeatedly maximize the log-likelihood function with respect to $\tau^2$ and $\beta$ until they converge.

\newpage
\section{Finalized Design Table}
A finalized design is a cross-array \citep{wu2011experiments} constructed over binary factor $z_4$ (the data type), continuous factors without linear constraint, $z_{1}, z_{2}, z_{3}$ (the weights ratio, the threshold of deviation, the proportion of mislabeling in training set), and continuous factors with linear constraint, $x_{1}, \ldots, x_{10}$ (the proportions of classes in training set), as shown in Table \ref{tab:design-table}. 

\begin{table}[h]
\begin{center}
\caption{Tables for settings of binary variable, settings of factors without constraint, and settings of factors with constraint.}\label{tab:design-table}
\vspace{1.5ex}
\begin{tabular}{lr}\hline\hline
Run & $z_{4}$
\\\hline
$1$ & $1$ \\
$2$ & $0$ 
\\ \hline\hline\\
\end{tabular}\\

\begin{tabular}{crrrrr}\hline\hline
Run & $z_{1}$ & $z_{2}$ & $z_{3}$\\\hline
$1$ & $333.5$ & $2.416$ & $0.190$\\
$2$ & $465.8$ & $1.151$ & $0.243$\\
... & ... & ... & ...\\  
$20$ & $0.005$ & $2.701$ & $0.345$\\\hline\hline\\
\end{tabular}\\
\begin{tabular}{crrrr}\hline\hline 
Run & $x_{1}$ & $x_{2}$ & ... & $x_{10}$\\\hline
$1$ & $0.02$ & $0.09$ & ... & $0.02$\\
$2$ & $0.14$ & $0.05$ & ... & $0.09$\\
... & ... & ... & ... & ...\\
$50$ & $0.05$ & $0.19$ & ... & $0.03$
\\ \hline\hline      
\end{tabular}
\end{center}
\end{table}

\newpage
\end{document}